\documentclass[final]{article}

\usepackage[nonatbib]{neurips_2021}
\usepackage{times}
\usepackage{epsfig}
\usepackage{graphicx}
\usepackage{amsmath}
\usepackage{amssymb}
\usepackage{amsthm}
\newtheorem{theorem}{Theorem}

\newtheorem{lemma}[theorem]{Lemma}

\newtheorem{definition}{Definition}
\usepackage{xcolor}
\usepackage{enumerate}
\newcommand{\indep}{\perp \!\!\! \perp}

\usepackage{graphicx}
\usepackage{amsmath}
\usepackage{wrapfig}
\usepackage{caption}
\usepackage{subcaption}
\usepackage{diagbox}
\usepackage{multirow}
\usepackage[toc,page]{appendix}
\usepackage[pagebackref=true,breaklinks=true,colorlinks,bookmarks=false,citecolor=blue,linkcolor=blue]{hyperref}
\usepackage{booktabs}

\usepackage{amsmath,amsfonts,bm}

\def\1{\bm{1}}

\DeclareMathAlphabet{\mathsfit}{\encodingdefault}{\sfdefault}{m}{sl}
\SetMathAlphabet{\mathsfit}{bold}{\encodingdefault}{\sfdefault}{bx}{n}

\renewcommand{\epsilon}{\varepsilon}

\newcommand{\ba}{\mathbf{a}}

\newcommand{\br}{\mathbf{r}}

\newcommand{\bt}{\mathbf{t}}

\newcommand{\bx}{\mathbf{x}}

\newcommand{\bz}{\mathbf{z}}

\newcommand{\bS}{\mathbf{S}}
\newcommand{\bT}{\mathbf{T}}

\newcommand{\bX}{\mathbf{X}}

\newcommand{\bZ}{\mathbf{Z}}

\newcommand{\btheta}{\bm{\theta}}

\newcommand{\cF}{\mathcal{F}}
\newcommand{\cG}{\mathcal{G}}
\newcommand{\cH}{\mathcal{H}}

\newcommand{\cL}{\mathcal{L}}

\newcommand{\cR}{\mathcal{R}}

\newcommand{\cT}{\mathcal{T}}

\newcommand{\bbR}{\mathbb{R}}

\usepackage{cleveref}
\crefname{equation}{Eq.}{Eq.}
\crefname{figure}{Fig.}{Fig.}
\crefname{table}{Tab.}{Tab.~}
\crefname{section}{Sec.}{Sec.~}
\crefname{algorithm}{Alg.}{Alg.~}
\crefname{thm}{Theorem}{Theorem~}
\crefname{appendix}{Appendix}{Appendix~}

\def\ie{\textit{i.e.,~}}
\def\eg{\textit{e.g.,~}}
\def\etc{\textit{etc}}

\def\wrt{\textit{w.r.t.~}}

\def\sota{state-of-the-art~}

\title{
Residual Relaxation for \\
Multi-view Representation Learning}

\author{%
  Yifei Wang$^1$ \quad Zhengyang Geng$^2$ \quad Feng Jiang$^2$ \quad Chuming Li$^3$ \\
  \textbf{Yisen Wang}$^2$\thanks{Corresponding author: Yisen Wang (yisen.wang@pku.edu.cn).} \quad \textbf{Jiansheng Yang}$^1$ \quad \textbf{Zhouchen Lin}$^{2,4}$
  \vspace{0.05in} \\
   $^1$ School of Mathematical Sciences, Peking University, China \\
   $^2$ Key Lab. of Machine Perception, School of EECS, Peking Univesity, China \\
   $^3$ School of Engineering, The University of Sydney, Australia \\
   $^4$ Pazhou Lab, Guangzhou 510330, China
   \vspace{0.05in} \\
}

\begin{document}

\maketitle

\begin{abstract}
Multi-view methods learn representations by aligning multiple views of the same image and their performance largely depends on the choice of data augmentation. In this paper, we notice that some other useful augmentations, such as image rotation, are harmful for multi-view methods because they cause a semantic shift that is too large to be aligned well. This observation motivates us to relax the exact alignment objective to better cultivate stronger augmentations. Taking image rotation as a case study, we develop a generic approach, Pretext-aware Residual Relaxation (Prelax), that relaxes the exact alignment by allowing an adaptive residual vector between different views and encoding the semantic shift through pretext-aware learning. Extensive experiments on different backbones show that our method can not only improve multi-view methods with existing augmentations, but also benefit from stronger image augmentations like rotation.

\end{abstract}

\section{Introduction}

Without access to labels, {self-supervised learning} relies on surrogate objectives to {extract meaningful representations from unlabeled data}, and the chosen surrogate objectives largely determine the quality and property of the learned representations~\cite{InfoMin,metzger2020evaluating}. {Recently, multi-view methods have become a dominant approach for self-supervised representation learning that} achieves impressive downstream performance, and many modern variants have been proposed \cite{InfoNCE,infomax, AMDIM,cmc,simclr,moco,simclrv2,mocov2,BYOL,simsiam}. Nevertheless, most multi-view methods can be abstracted and summarized as the following pipeline: for each input $\bx$, we apply several (typically two) random augmentations to it, and learn to align these different ``views'' ($\bx_1,\bx_2,\dots$) of $\bx$ by minimizing their distance in the representation space.

In multi-view methods, the pretext, \ie image augmentation, has a large effect on the final performance. Typical choices include image re-scaling, cropping, color jitters, \emph{etc}~\cite{simclr}. However, we find that some augmentations, for example, image rotation, is seldom utilized in \sota multi-view methods. 
Among these augmentations, Figure \ref{fig:augmentation-comparison} shows that rotation causes severe accuracy drop in a standard supervised model.
Actually, image rotation is a stronger augmentation that largely affects the image semantics, and as a result, enforcing exact alignment of two different rotation angles could degrade the representation ability in existing multi-view methods. Nevertheless, it does not mean that strong augmentations cannot provide useful semantics for representation learning. In fact, rotation is known as an effective signal for predictive  learning \cite{rotation,colorization,jigsaw}. Differently, predictive methods learn representations by predicting the pretext (\eg rotation angle) from the corresponding view. In this way, the model is encouraged to encode pretext-aware image semantics, which also yields good representations. 

To summarize, strong augmentations like rotation carry meaningful semantics, while being harmful for existing multi-view methods due to large semantic shift. To address this dilemma, in this paper, we propose a generic approach that generalizes multi-view methods to cultivating stronger augmentations. Drawing inspirations from the soft-margin SVM, we propose \emph{residual alignment}, which relaxes the exact alignment in multi-view methods by incorporating a residual vector between two views. Besides, we develop a predictive loss for the residual to ensure that it encodes the semantic shift between views (\eg image rotation). 
We name this technique as Pretext-aware REsidual ReLAXation (Prelax), and an illustration is shown in Figure \ref{fig:alignment-comparison}. Prelax serves as a generalized multi-view method that is adaptive to large semantic shift and combines image semantics extracted from both pretext-invariant and pretext-aware methods. 
We summarize our contributions as follows:
\begin{itemize}
    \item We propose a generic technique, Pretext-aware Residual Relaxation (Prelax), that generalizes multi-view representation learning to benefit from stronger image augmentations.
    \item Prelax not only extracts pretext-invariant features as in multi-view methods, but also encodes pretext-aware features into the pretext-aware residuals. Thus, it can serve as a unified approach to bridge the two existing methodologies for representation learning.
    \item Experiments show that Prelax can bring significant improvement over both multi-view and predictive methods on a wide range of benchmark datasets. 

\end{itemize}

\begin{figure}
    \centering
    \begin{subfigure}{.45\textwidth}
        \centering
        \includegraphics[width=\linewidth]{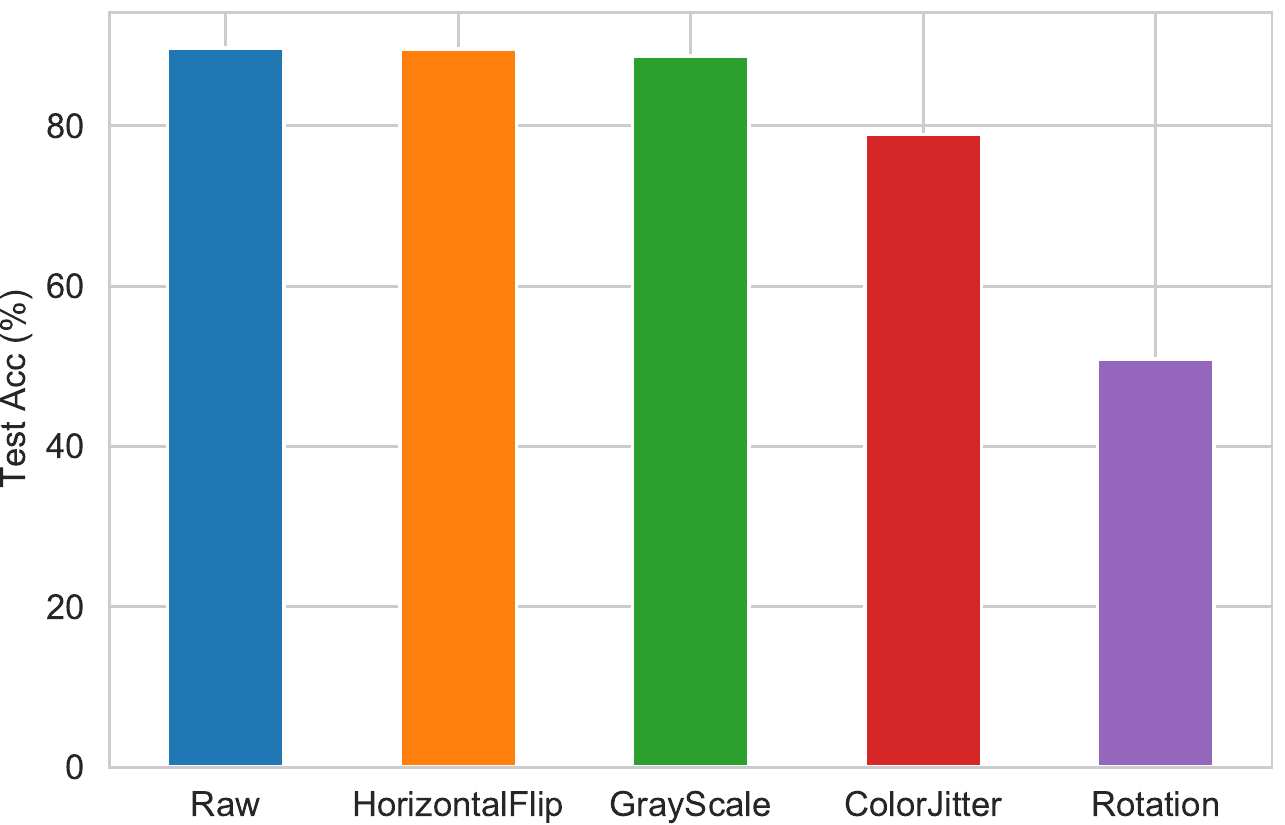}
        \caption{Comparison of augmentations.}
        \label{fig:augmentation-comparison}
    \end{subfigure}
    \hfill
    \begin{subfigure}{.51\textwidth}
        \centering
        \includegraphics[width=\linewidth]{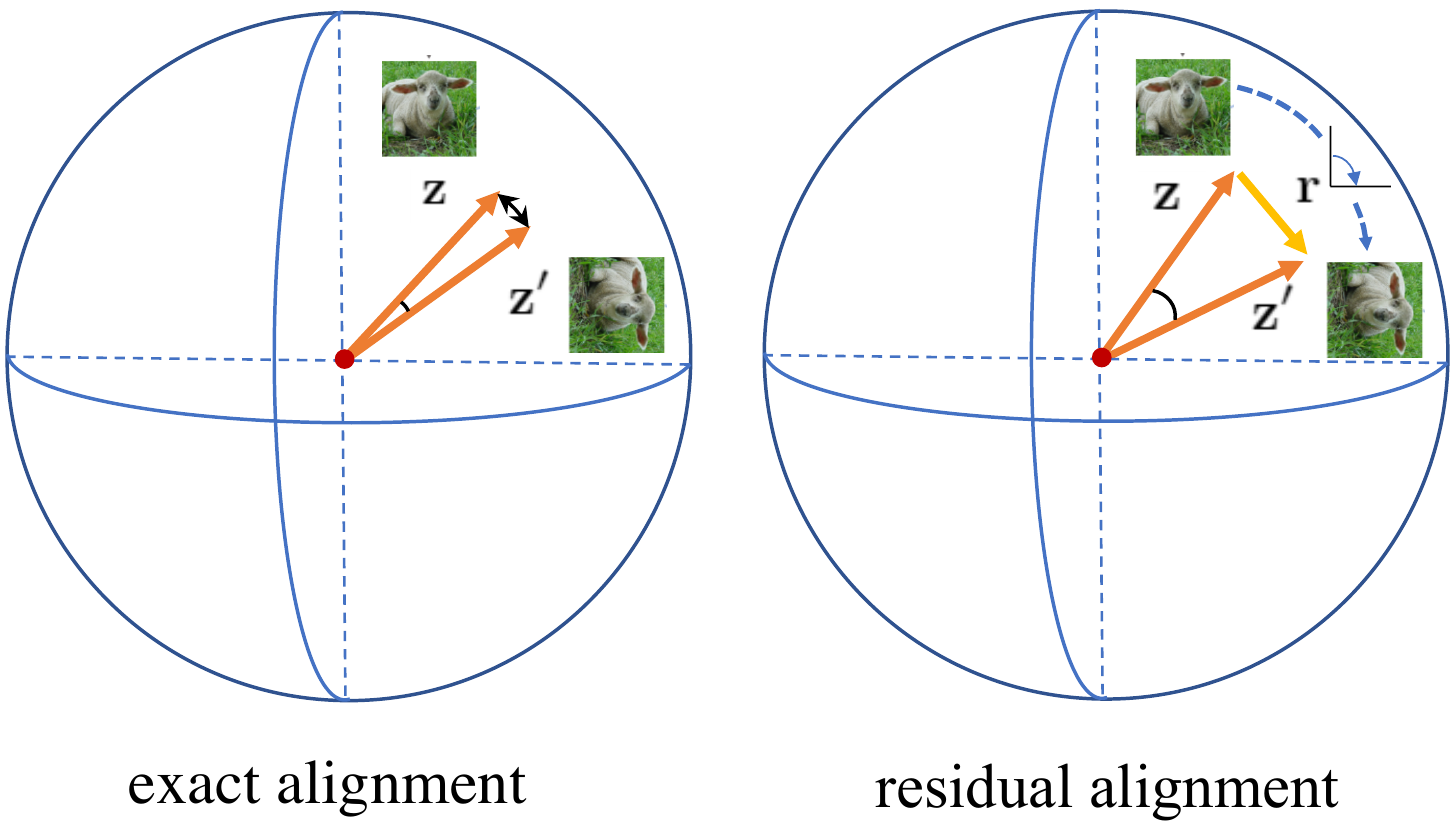}
        \caption{A toy example of residual relaxation.}
        \label{fig:alignment-comparison}
    \end{subfigure}
    \caption{{Left: the effect of different augmentations of CIFAR-10 test images with a supervised model (trained without using any data augmentation, more details in Appendix \ref{sec:appendix-experimental-details}).}
    Right: an illustration of the exact alignment objective of multi-view methods ($\bz'\rightarrow\leftarrow\bz$) and the relaxed residual alignment of our Prelax ($\bz'-\br\rightarrow\leftarrow\bz$). As the rotation largely modifies the image semantics, our Prelax adopts a rotation-aware residual vector $\br$ to bridge the representation of two different views.}
    \label{fig:prelax-diagrams}
\end{figure}

\section{Related Work}

\textbf{Multi-view Learning.} {Although multi-view learning could refer to a wider literature \cite{li2018survey}, here we restrict our discussions to the context of Self-Supervised Learning (SSL), where multi-view methods learn representations by aligning multiple views of the same image generated through random data augmentation \cite{AMDIM}. }
There are two kinds of methods to keep the representations well separated: contrastive methods, which achieve this by maximizing the difference between different samples \cite{simclr,moco}, and similarity-based methods, which prevent representation collapse via implicit mechanisms like predictor and gradient stopping ~\cite{BYOL,simsiam}.
Although having lots of modern variants, multi-view methods share the same methodology, that is to extract features that are \emph{invariant} to the predefined augmentations, \ie pretext-invariant features \cite{pirl}. 

\textbf{Predictive Learning.} 
Another thread of methods is to learn representations by predicting self-generated surrogate labels. 
Specifically, it applies a transformation (\eg image rotation) to the input image and requires the learner to predict properties of the transformation (\eg the rotation angle) from the transformed images. 
As a result, the extracted image representations are encouraged to become aware of the applied pretext (\eg image rotation). Thus, we also refer to them as \emph{pretext-aware methods}. The pretext tasks can be various, to name a few, Rotation~\cite{rotation}, Jigsaw~\cite{jigsaw}, Relative Path Location~\cite{rel_loc_pred}, Colorization~\cite{colorization}.

\textbf{Generalized Multi-view Learning.} Although there are plenty of works on each branch, how to bridge the two methodologies remains under-explored. Prior to our work, there are only a few works on this direction. Some directly combine AMDIM (pretext-invariant) \cite{AMDIM} and Rotation (pretext-aware) \cite{rotation} objectives \cite{feng2019self}. However, a direct combination of the two contradictory objectives may harm the final representation.
LooC~\cite{looc} proposes to separate the embedding space to several parts, where each subspace learns local invariance \wrt a specific augmentation. But this is achieved at the cost of limiting the representation flexibility of each pretext to the predefined subspace. Different from them, our proposed Prelax provides a more general solution by allowing an adaptive residual vector to encode the semantic shift. In this way, both kinds of features are encoded in the same representation space.

\section{The Proposed Pretext-aware Residual Relaxation (Prelax) Method}

\subsection{Preliminary}
\textbf{Problem Formulation.}
Given unlabeled data $\{\bx_i\}$, unsupervised representation learning aims to learn an encoder network $\cF_{\btheta}$ that extracts meaningful low-dimensional representations $\bz\in\bbR^{d_z}$ from high-dimensional input images $\bx\in\bbR^{d_x}$. The learned representation is typically evaluated on a downstream classification task by learning a linear classifier with labeled data $\{\bx_i,y_i\}$.

\textbf{Multi-view Representation Learning.} For an input image $\bx\in\bbR^{d_x}$, we can generate a different view  by data augmentation, $\bx'=t(\bx)$, where $t\in\cT$ refers to a randomly drawn augmentation operator from the pretext set $\cT$. Then, the transformed input $\bx'$ and the original input $\bx$ are passed into an online network $\cF_{\btheta}$ and a target network $\cF_{\bm \phi}$, 
respectively. Optionally, the output of the online network is further processed by an MLP predictor network $\cG_{\btheta}$, to match the output of the target network. As two different views of the same image (\ie positive samples), $\bx$ and $\bx'$ should have similar representations, so we align their representations with the following similarity loss,
\begin{equation}
    \cL_{\rm sim}(\bx',\bx;\btheta)=\Vert\cG_{\btheta}\left(\cF_{\btheta}(\bx')\right) - \cF_{\bm\phi}(\bx)\Vert_2^2.
    \label{eqn:similarity-loss}
\end{equation}
The representations, \eg $\bz=\cF_{\btheta}(\bx)$, are typically projected to a unit spherical ball before calculating the distance $(\bz/\Vert\bz\Vert_2)$, which makes the $\ell_2$ distance equivalent to the cosine similarity~\cite{simclr}.

\textbf{Remark.} Aside from the similarity loss between positive samples, contrastive  methods~\cite{InfoNCE,infomax,cmc,simclr} further encourage representation uniformity with an additional regularization minimizing the similarity between input and an independently drawn negative sample. Nevertheless, some recent works find that the similarity loss alone already suffices \cite{BYOL,simsiam}. In this paper, we mainly focus on improving the alignment between positive samples in the similarity loss. It can also be easily extended to contrastive methods by considering the dissimilarity regularization.

\subsection{Objective Formulation}
\label{sec:method}

As we have noticed, the augmentation sometimes may bring a certain amount of semantic shift.
Thus, enforcing exact alignment of different views may hurt the representation quality, particularly when the data augmentation is too strong for the positive pairs to be matched exactly. 
Therefore, we need to relax the exact alignment in Eq.~\eqref{eqn:similarity-loss} to account for the semantic shift brought by the data augmentation.

\textbf{Residual Relaxed Similarity Loss.}
Although the representations may not align exactly, \ie $\bz'\neq\bz$, however, the \emph{representation identity} will always hold: $\bz' - (\bz'-\bz) = \bz$, where $\bz'-\bz$ represents the shifted semantics by augmentation. 
This makes this identity a proper candidate for multi-view alignment under various augmentations as long as the shifted semantic is taken into consideration.

Specifically, we replace the exact alignment (denoted as $\rightarrow\leftarrow$) in the similarity loss (Eq.~\eqref{eqn:similarity-loss}) by the proposed \emph{identity alignment}, \ie
\begin{equation}
\cG_{\btheta}(\bz'_{\btheta})\rightarrow\leftarrow\bz_{\bm\phi} 
\quad\Rightarrow\quad
\cG_{\btheta}(\bz'_{\btheta})-\cG_{\btheta}(\br)\rightarrow\leftarrow\bz_{\bm\phi},
 \label{eqn:identity-alignment}
\end{equation}
where we include a residual vector $\br\overset{\Delta}{=}\bz'_{\btheta} - \bz_{\btheta}=\cF_{\btheta}(\bx')-\cF_{\btheta}(\bx)$ to represent the difference on the representations. 
To further enable a better tradeoff between the exact and identity alignments, we have the following \emph{residual alignment}:
\begin{equation}
 \cG_{\btheta}(\bz'_{\btheta})-\alpha \cG_{\btheta}(\br)
 \rightarrow\leftarrow\bz_{\bm\phi},
\label{eqn:unified-alignment}
\end{equation}
where $\alpha\in[0,1]$ is the interpolation parameter. 
When $\alpha=0$, we recover the exact alignment; when $\alpha=1$, we recover the identity alignment.
We name the corresponding learning objective as the Residual Relaxed Similarity (R2S) loss, which minimizes the squared $\ell_2$ distance among two sides:
\begin{equation}
\cL^\alpha_{\rm R2S}(\bx',\bx;\btheta)=
\Vert\cG_{\btheta}\left(\cF_{\btheta}(\bx')\right) -\alpha \cG_{\btheta}(\br)- \cF_{\bm\phi}(\bx)\Vert_2^2.
\label{eqn:relaxed-similarity-loss}
\end{equation}

\textbf{Predictive Learning (PL) Loss.} 
To ensure the relaxation works as expected, the residual $\br$ should encode the semantic shift caused by the augmentation, \ie the pretext. Inspired by predictive learning \cite{rotation}, we utilize the residual to predict the corresponding augmentation for its pretext-awareness.
In practice, the assigned parameters for the random augmentation $\bt$ can be generally divided into the discrete categorical variables $\bt^d$ (\eg flipping or not, graying or not), and the continuous variables $\bt^c$ (\eg scale, ratio, jittered brightness). Thus, we learn a PL predictor $\cH_{\btheta}$ to predict $(\bt^d,\bt^c)$ with cross entropy loss (CE) and mean square error loss (MSE), respectively:
\begin{equation}
    \cL_{\rm PL}(\bx',\bx,\bt;\btheta)=\operatorname{CE}(\cH^d_{\btheta}(\br),\bt^d) + \Vert\cH^c_{\btheta}(\br)-\bt^c\Vert_2^2.
    \label{eqn:PL-loss}
\end{equation}

\textbf{Constraint on the Similarity.} 
Different from the exact alignment, the residual vector can be unbounded, \ie the difference between views can be arbitrarily large. This is not reasonable as the two views indeed share many common semantics. Therefore, we should utilize this prior knowledge to prevent the bad cases under residual similarity and add the following constraint  
\begin{equation}
    \begin{aligned}
    \cL_{\rm sim} = \Vert\cG_{\btheta}\left(\cF_{\btheta}(\bx')\right) - \cF_{\bm\phi}(\bx)\Vert_2^2\leq\varepsilon,
    \end{aligned}
\end{equation}
where $\varepsilon$ denotes the maximal degree of mismatch allowed between positive samples.

\textbf{The Overall Objective of Prelax.} By combining the three components above, we can reliably encode the semantic shift between augmentations while ensuring a good alignment between views:
\begin{equation}
\begin{gathered}
\min_{\btheta}\ \cL^\alpha_{\rm R2S}(\bx',\bx;\btheta) + \gamma \cL_{\rm PL}(\bx',\bx;\btheta), \\
s.t.\quad \Vert\cG_{\btheta}\left(\cF_{\btheta}(\bx')\right) - \cF_{\bm\phi}(\bx)\Vert_2^2\leq\varepsilon.
\end{gathered}
\label{eqn:Prelax}
\end{equation}
For simplicity, we transform it into a Lagrangian objective with a fixed multiplier $\beta\geq0$, and obtain the overall Pretext-aware REsidual ReLAXation (Prelax) objective,
\begin{equation}
\cL^\alpha_{\rm R2S}(\bx',\bx;\btheta)+\gamma \cL_{\rm PL}(\bx',\bx;\btheta)+\beta\cL_{\rm sim}(\bx',\bx;\btheta),
\end{equation}
where $\alpha$ tradeoffs between the exact and identity alignments, $\gamma$ adjusts the amount of pretext-awareness of the residual, and $\beta$ controlls the degree of similarity between positive pairs. An illustrative diagram of the Prelax objective is shown in Figure \ref{fig:prelax-diagrams}. 
\begin{figure}[!t]
\centering
\includegraphics[width=.6\linewidth]{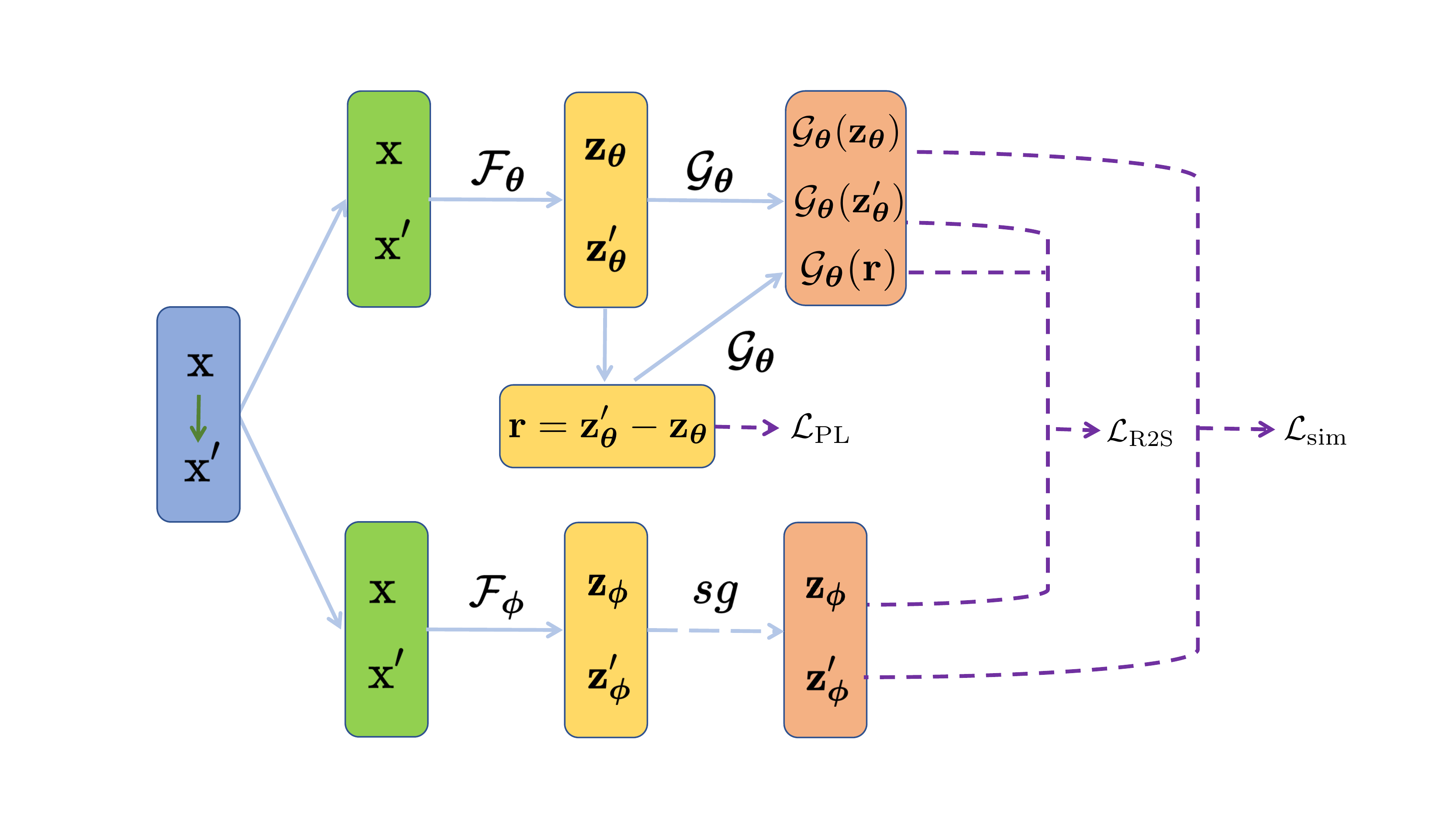}
\caption{A diagram of our proposed Prelax objective. 
{An image $\bx$ is firstly augmented as $\bx'$. Then the positive pair $(\bx,\bx')$, is processed by the online network $\cF_{\btheta}$ and the target network $\cF_{\bm\phi}$, respectively. Output of the online network is further processed by the target network $\cG_{\btheta}$, and the gradient of $\cF_{\bm\phi}$ is detached, \ie $\operatorname{stop\_grad}$, denoted as $\operatorname{sg}$. Then the outputs are used to compute the three objectives, $\cL_{\rm R2S}$ (Eq.~\ref{eqn:relaxed-similarity-loss}),  $\cL_{\rm PL}$ (Eq.~\ref{eqn:PL-loss}), and $\cL_{\rm sim}$ (Eq.~\ref{eqn:similarity-loss}) in the Prelax objective (Eq.~\ref{eqn:Prelax}).}
}
\label{fig:prelax-objective}
\end{figure}

\textbf{Discussions. } In fact, there are other alternatives to relax the exact alignment. For example, we can utilize a margin loss 
\begin{equation}
    \cL_{\rm margin}(\bx',\bx;\btheta)=\max(\Vert \cG_{\btheta}\left(\cF_{\btheta}(\bx')\right) - \cF_{\bm\phi}(\bx)\Vert_2^2-\eta,0),
    \label{eqn:margin-loss}
\end{equation}
where $\eta > 0$ is a threshold for the mismatch tolerance. However, it has two main drawbacks: 1) as each image and augmentation have different semantics, it is hard to choose a universal threshold for all images; and 2) the representation keeps shifting along with the training progress, making it even harder to maintain a proper threshold dynamically.
Thus, a good relaxation should be adaptive to the training progress and the aligning of different views. While our Prelax adopts \emph{pretext-aware residual vector}, which is learnable, flexible, and semantically meaningful.

\subsection{Theoretical Analysis}

As Prelax encodes both pretext-invariant and pretext-aware features, it can be semantically richer than both multi-view learning and predictive learning. Following the information-theoretic framework developed by \cite{tsai2020self}, we show that Prelax provably enjoys better downstream performance.

We denote the random variable of input as $\bX$ and learn a representation $\bZ$ through a deterministic encoder $\cF_{\btheta}$: $\bZ=\cF_{\btheta}(\bX)$\footnote{We use capitals to denote the random variable, \eg $\bX$, and use lower cases to denote its outcome, \eg $\bx$.}.
The representation $\bZ$ is evaluated for a downstream task $\bT$ by learning a classifier on top of $\bZ$. From an information-theoretic learning perspective, a desirable algorithm should maximize the Mutual Information (MI) between $\bZ$ and $\bT$, \ie $I(\bZ;\bT)$ \cite{cover1999elements}.
Supervised learning on task $\bT$ can learn representations by directly maximizing $I(\bZ;\bT)$.
Without access to the labels $\bT$, unsupervised learning resorts to maximizing $I(\bZ;\bS)$, where $\bS$ denotes the surrogate signal $\bS$ designed by each method. 
Specifically, multi-view learning matches $\bZ$ with the randomly augmented view, denoted as $\bS_v$; while predictive learning uses $\bZ$ to predict the applied augmentation, denoted as $\bS_a$. In Prelax, as we combine both semantics, we actually maximize the MI \wrt their joint distribution, \ie $I(\bZ;\bS_v,\bS_a)$. We denote the representations learned by supervised learning, multi-view learning, predictive learning, and Prelax as $\bZ_{\rm sup}, \bZ_{\rm mv},\bZ_{\rm PL}, \bZ_{\rm Prelax}$, respectively.

\begin{theorem}
Assume that by maximizing the mutual information, each method can retain all information in $\bX$ about the learning signal $\bS$ (or $\bT$), \ie $I(\bX;\bS)=\max_{\bZ} I(\bZ;\bS)$. Then we have the following inequalities on their task-relevant information $I(\bZ;\bT)$:
\begin{gather}
I(\bX;\bT)=I(\bZ_{\rm sup};\bT)\geq I(\bZ_{\rm Prelax};\bT)\geq\max(I(\bZ_{\rm mv};\bT),I(\bZ_{\rm PL};\bT)).
\end{gather}
\label{thm:task-information}
\end{theorem}
\begin{theorem}
Further assume that $\bT$ is a $K$-class categorical variable. In general, we have the upper bound $u^e$ on the downstream Bayes errors $P^{e}:=\mathbb{E}_\bz\left[1-\max _{\bt \in \bT} P\left(\bT=\bt|\bz\right)\right]$,
\begin{equation}
\bar P^e\leq u^e :=\log 2 + P^e_{\rm sup}\cdot\log K + I(\bX;\bT|\bS).
\end{equation}
where $\bar P^e=\operatorname{Th}(P^e)=\min \{\max \{P^e, 0\}, 1-1 /K\}$ denotes the thresholded Bayes error. Accordingly, we have the following inequalities on the upper bounds for different unsupervised methods, 
\begin{equation}
    u^e_{\rm sup}\leq u^e_{\rm Prelax} \leq \min(u^e_{\rm mv}, u^e_{\rm PL}) \leq \max(u^e_{\rm mv}, u^e_{\rm PL}).
\end{equation}
\label{thm:bayes-errors}
\end{theorem}
Theorem \ref{thm:task-information} shows that Prelax extracts more task-relevant information than multi-view and predictive methods, and Theorem \ref{thm:bayes-errors} further shows that Prelax has a tighter upper bound on the downstream Bayes error. Therefore, Prelax is indeed theoretically superior to previous unsupervised methods by utilizing both pretext-invariant and pretext-aware features. Proofs are in Appendix \ref{sec:appendix-proofs}.

\section{Practical Implementation}

In this part, we present three practical variants of Prelax to generalize existing multi-view backbones: 1) one with existing multi-view augmentations (Prelax-std); 2) one with a stronger augmentation, image rotation (Prelax-rot); and 3) one with previous two strategies (Prelax-all).

\subsection{Backbone} 

BYOL~\cite{BYOL} and SimSiam~\cite{simsiam} are both similarity-based methods and they differ mainly in the design of the target network $\cF_{\bm\phi}$. BYOL~\cite{BYOL}  utilizes momentum update of the target parameters $\bm\phi$ from the online parameters $\btheta$, \ie ${\bm\phi} \leftarrow \tau{\bm\phi} + (1-\tau)\btheta$,
where $\tau\in[0,1]$ is the target decay rate. While SimSiam~\cite{simsiam} simply regards the (stopped-gradient) online network as the target network, \ie ${\bm\phi}\leftarrow\operatorname{sg}(\btheta)$. 
We mainly take SimSiam for discussion and our analysis also applies to BYOL.

For a given training image $\bx$, SimSiam draws two random augmentations $(t_1, t_2)$ and get two views $(\bx_1,\bx_2)$, respectively. Then, SimSiam maximizes the similarity of their representations with a dual objective, where the two views can both serve as the input and the target to each other,
\begin{equation}
\cL_{\rm Simsiam}(\bx;\btheta)=\Vert\cG_{\btheta}(\cF_{\btheta}(\bx_1))-\cF_{\bm\phi}(\bx_2)\Vert_2^2 + \Vert\cG_{\btheta}(\cF_{\btheta}(\bx_2))-\cF_{\bm\phi}(\bx_1)\Vert_2^2.
\label{eqn:baseline-dual-objective}
\end{equation}

\subsection{Prelax-std}
To begin with, we can directly generalize the baseline method with our Prelax method under existing multi-view augmentation strategies \cite{simclr,BYOL}. For the same positive pair $(\bx_1,\bx_2)$, we can calculate their residual vector $\br_{12}=\cF_{\btheta}(\bx_1)-\cF_{\btheta}(\bx_2)$ and use it for the R2S loss (Eq.~\eqref{eqn:relaxed-similarity-loss}) 
\begin{equation}
\cL^{\alpha}_{\rm R2S}(\bx_1,\bx_2;\btheta)=\Vert\cG_{\btheta}\left(\cF_{\btheta}(\bx_1)\right)-\alpha\cG_{\btheta}(\br_{12}) - \cF_{\bm\phi}(\bx_2)\Vert_2^2.
\label{eqn:Prelax-two-R2S-loss}
\end{equation}
{We note that there is no difference in using $\br_{12}$ or $\br_{21}$ as the two views are dual.} Then, we can adopt the similarity loss in the reverse direction as our similarity constraint loss,
\begin{equation}
    \cL_{\rm sim}(\bx_2,\bx_1;\btheta)=\Vert\cG_{\btheta}\left(\cF_{\btheta}(\bx_2)\right) - \cF_{\bm\phi}(\bx_1)\Vert_2^2.
    \label{eqn:Prelax-std-reverse-similarity-loss}
\end{equation}
At last, we use the residual $\br_{12}$ for the PL loss to predict the augmentation parameters of $\bx_1$, \ie $\bt_1$, because $\br_{12}=\bz_1-\bz_2$ directs towards $\bz_1$.
Combining the three losses above, we obtain our Prelax-std objective, 
\begin{equation}
\cL_{\rm Prelax-std}(\bx;\btheta)=\cL^{\alpha}_{\rm R2S}(\bx_1,\bx_2;\btheta)+\gamma\cL_{\rm PL}(\bx_1,\bx_2,\bt_1;\btheta)+\beta\cL_{\rm sim}(\bx_2,\bx_1;\btheta).
\label{eqn:Prelax-std-overall-objective}
\end{equation}

\subsection{Prelax-rot}

As mentioned previously, with our residual relaxation we can benefit from stronger augmentations that are harmful for multi-view methods. Here, we focus on the image rotation example and propose the Prelax-rot objective with rotation-aware residual vector. To achieve this, we further generalize existing dual-view methods by incorporating a \emph{third} rotation view. 

Specifically, given two views $(\bx_1,\bx_2)$ generated with existing multi-view augmentations, we additionally draw a random rotation angle $a\in\cR=\{0^\circ,90^\circ,180^\circ,270^\circ\}$ and apply it to rotate $\bx_1$ clockwise, leading to the third view $\bx_3$.
Note that the only difference between $\bx_3$ and $\bx_1$ is the rotation semantic $a$. Therefore, if we substitute $\bx_1$ with $\bx_3$ in the similarity loss, we should add a rotation-aware residual $\br_{31}=\bz_3-\bz_1$ to bridge the gap.
Motivated by this analysis, we propose the Rotation Residual Relaxation Similarity (R3S) loss, 
\begin{equation}
    \cL^\alpha_{\rm R3S}(\bx_{1:3};\btheta)=\Vert\cG_{\btheta}(\cF_{\btheta}(\bx_3))-\alpha\cG_{\btheta}(\br_{31})-\cF_{\bm\phi}(\bx_2)\Vert_2^2.
\end{equation}
which replace $\cG_{\btheta}(\cF_{\btheta}(\bx_1))$ by its rotation-relaxed version $\cG_{\btheta}(\cF_{\btheta}(\bx_3))-\alpha\cG_{\btheta}(\br_{31})$ in the similarity loss. {Comparing the R2S loss (Eq.~\ref{eqn:Prelax-two-R2S-loss}) and the R3S loss, we note that the relaxation of the R2S loss accounts for all the semantic shift between $\bx_1$ and $\bx_2$, while that of the R3S loss only accounts for the rotation augmentation between $\bx_1$ and $\bx_3$. 
Therefore, we could use the residual $\br_{31}$ to predict the rotation angle $a$ with the following RotPL loss for its rotation-awareness:}
\begin{equation}
    \cL^{\rm rot}_{\rm PL}(\bx_1,\bx_3,\ba;\btheta)=\operatorname{CE}(\cH_{\btheta}(\br_{31}),a).
\end{equation}
Combining with the similarity constraint,  we obtain the Prelax-rot objective:
\begin{equation}
\cL_{\rm Prelax-rot}(\bx;\btheta)=\cL^{\alpha}_{\rm R3S}(\bx_{1:3};\btheta)+\gamma\cL^{\rm rot}_{\rm PL}(\bx_1,\bx_3,a;\btheta)+\beta\cL_{\rm sim}(\bx_2,\bx_1;\btheta).
\label{eqn:Prelax-rot-overall-objective}
\end{equation}

\subsection{Prelax-all} 
We have developed Prelax-std that cultivates existing multi-view augmentations and Prelax-rot that incorporates image rotation. Here, we further utilize both existing augmentations and image rotation by combining the two objectives together, denoted as Prelax-all:
{
\begin{equation}
\begin{aligned}
\cL_{\rm Prelax-all}(\bx;\btheta)=&\frac{1}{2}\left(\cL^{\alpha_1}_{\rm R2S}(\bx_1,\bx_2;\btheta)+\cL^{\alpha_2}_{\rm R3S}(\bx_{1:3};\btheta)\right)
+\frac{\gamma_1}{2}\cL_{\rm PL}(\bx_1,\bx_2,\bt_1;\btheta)\\
&+\frac{\gamma_2}{2}\cL^{\rm rot}_{\rm PL}(\bx_1,\bx_3,a;\btheta)+\beta\cL_{\rm sim}(\bx_2,\bx_1;\btheta),
\end{aligned}
\label{eqn:Prelax-all-overall-objective}
\end{equation}
}
where $\alpha_1,\alpha_2,\gamma_1,\gamma_2$ denotes the  coefficients for R2S, R3S, PL and RotPL losses, respectively.

{
\subsection{Discussions} %
Here we design three practical versions as different implementations of our generic framework of residual relaxation. Among them, Prelax-std focuses on further cultivating existing augmentation strategies, Prelax-rot is to incorporate the stronger (potentially harmful) rotation augmentation, while Prelax-all combines them all. Through the three versions, we demonstrate the wide applicability of Prelax as a generic framework. As for practical users, they could also adapt Prelax to their own application by incorporating specific domain knowledge. In this paper, as we focus on natural images, we take rotation as a motivating example as it is harmful for natural images. Nevertheless, rotation is not necessarily harmful in other domains, \eg medical images. Instead, random cropping could instead be very harmful for medical images as the important part could lie in the corner. In this scenario, our residual relaxation mechanism could also be used to encode the semantic shift caused by cropping and alleviate its bad effects.}

\section{Experiments}

\textbf{Datasets.} 
Due to computational constraint, we carry out experiments on a range of medium-sized real-world image datasets, including well known benchmarks like CIFAR-10 \cite{cifar}, CIFAR-100 \cite{cifar}, and two ImageNet variants: Tiny-ImageNet-200 (200 classes with image size resized to 32$\times$32) \cite{tinyImageNet} and ImageNette (10 classes with image size 128$\times$128)\footnote{\url{https://github.com/fastai/imagenette}}.

\textbf{Backbones.} 
As Prelax is designed to be a generic method for generalizing existing multi-view methods, we implement it on two different multi-view methods, SimSiam~\cite{simsiam} and BYOL~\cite{BYOL}. 
Specifically, we notice that SimSiam reported results on CIFAR-10, while the official code of BYOL included results on ImageNette. For a fair comparison, we evaluate SimSiam and its Prelax variant on CIFAR-10, and evaluate BYOL and its Prelax variant on ImageNette. In addition, we evaluate SimSiam and its Prelax variant on two additional datasets CIFAR-100 and Tiny-ImageNet-200, which are more challenging because they include a larger number of classes.  
For computational efficiency, we adopt the ResNet-18 \cite{he2016deep} backbone (adopted in SimSiam \cite{simsiam} for CIFAR-10) to benchmark our experiments. For a comprehensive comparison, we also experiment with larger backbones, like ResNet-34 \cite{he2016deep}, and the results are included in Appendix \ref{sec:appendix-larger-backbone}. 

\textbf{Setup.} For Prelax-std, we use the same data augmentations as SimSiam~\cite{simclr,simsiam} (or BYOL \cite{BYOL}), including RandomResizedCrop, RandomHorizontalFlip, ColorJitter, and RandomGrayscale, \etc~ using the PyTorch notations. For Prelax-rot and Prelax-all, we further apply a random image rotation at last of the transformation, where the angles are drawn randomly from $\{0^\circ,90^\circ,180^\circ,270^\circ\}$. To generate targets for the PL objective in Prelax, for each image, we collect the assigned parameters in each random augmentation, such as crop centers, aspect ratios, rotation angles, \etc. 
More details can be found in Appendix \ref{sec:appendix-experimental-details}.

\textbf{Training.} For SimSiam and its Prelax variants, we follow the same hyperparameters in \cite{simsiam} on CIFAR-10. Specifically, we use ResNet-18 as the backbone network, followed by a 3-layer projection MLP, whose hidden and output dimension are both 2048. The predictor is a 2-layer MLP whose hidden layer and output dimension are 512 and 2048 respectively. We use SGD for pre-training with batch size 512, learning rate 0.03, momentum 0.9, weight decay $5\times10^{-4}$, and cosine decay schedule~\cite{cosine} for 800 epochs. For BYOL and its Prelax variants, we also adopt the ResNet-18 backbone, and the projector and predictor are 2-layer MLPs whose hidden layer and output dimension are 256 and 4096 respectively. Following the default hyper-parameters on ImageNette\footnote{\url{https://github.com/deepmind/deepmind-research/tree/master/byol}}, we use LARS optimizer \cite{lars} to train 1000 epochs with batch size 256, learning rate 2.0, weight decay $1\times10^{-6}$ while excluding the biases and batch normalization parameters from both LARS adaptation and weight decay. For the target network, the exponential moving average parameter $\tau$ starts from $\tau_\text{base} = 0.996$ and increases to 1 during training.
As for the Prelax objective, we notice that sometimes, adopting a reverse residual $\br_{21}$ in the R2S loss (Eq.~\eqref{eqn:Prelax-two-R2S-loss}) can bring slightly better results, which could be due to the swapped prediction in SimSiam's dual objective (Eq.~\eqref{eqn:baseline-dual-objective}). Besides, a na\"ive choice of Prelax coefficients already works well: $\alpha=1,\beta=1,\gamma=0.1$ for Prelax-std and Prelax-rot, and $\alpha_1=\alpha_2=1,\beta=1,\gamma_1=\gamma_2=0.1$ for Prelax-all. More discussion about the hyper-parameters of Prelax can be found in Appendix \ref{sec:appendix-sensitivity}.

\textbf{Evaluation.} After unsupervised training, we evaluate the backbone network by fine-tuning a linear classifier on top of its representation with other model weights held fixed. For SimSiam and its Prelax variants, the linear classifier is trained on labeled data from scratch using SGD with batch size 256, learning rate 30.0, momentum 0.9 for 100 epochs. The learning rate decays by 0.1 at the 60-th and 80-th epochs. For BYOL and its Prelax variants, we use SGD with Nesterov momentum over 80 epochs using
batch size 25, learning rate 0.5 and momentum 0.9. 
Besides the in-domain linear evaluation, we also evaluate its transfer learning performance on the out-of-domain data by learning a linear classifier on the labeled target domain data.

\begin{table}\centering
    \caption{Linear evaluation on CIFAR-10 (a) and ImageNette (b) with ResNet-18 backbone. TTA: Test-Time Augmentation. }
\begin{subtable}[h]{0.47\textwidth}
    \centering
    \caption{CIFAR-10.}
    \begin{tabular}{lc}
    \toprule
        Method & Acc. (\%) \\
    \midrule
    Supervised \cite{he2016deep}  (re-produced) & 95.0 \\
    \midrule
    Rotation~\cite{rotation}  (re-produced)  & 88.3 \\
        BYOL~\cite{BYOL} (re-produced) & 91.1 \\
        SimCLR~\cite{simclr} & 91.1 \\
        SimSiam~\cite{simsiam} & 91.8 \\
        \midrule
        \textbf{SimSiam + Prelax}  & \textbf{93.4} \\
     \bottomrule
    \end{tabular}
    \label{tab:CIFAR-10-previous-methods}
\end{subtable}
\hfill
\begin{subtable}[h]{0.47\textwidth}
    \centering
    \caption{ImageNette.}
    \begin{tabular}{lc}
    \toprule
        Method & Acc. (\%) \\
    \midrule
    Supervised  & 91.0 \\
    Supervised + TTA & 92.2 \\
    \midrule
        BYOL~\cite{BYOL} (ResNet-18)  & 91.9 \\
        BYOL~\cite{BYOL} (ResNet-50) & 92.3 \\
        \midrule
        \textbf{BYOL + Prelax (ResNet-18)}  & \textbf{92.6} \\
     \bottomrule
    \end{tabular}
    \label{tab:imagette-previous-methods}
\end{subtable}
\end{table}

\subsection{Performance on Benchmark Datasets}

\textbf{CIFAR-10.} 
In Table \ref{tab:CIFAR-10-previous-methods}, we compare Prelax against previous multi-view methods (SimCLR \cite{simclr}, SimSiam \cite{simsiam}, and BYOL \cite{BYOL}) and predictive methods (Rotation \cite{rotation}) on CIFAR-10. We can see that multi-view methods are indeed better than predictive ones. Nevertheless, predictive learning alone (\eg Rotation) achieves quite good performance, indicating that pretext-aware features are also very useful. By encoding both pretext-invariant and pretext-aware features, Prelax outperforms previous methods by a large margin, and achieve \sota performance on  CIFAR-10. A comparison of the learning dynamics between SimSiam and Prelax can be found in Appendix \ref{sec:appendix-learning-dynamics}.

\textbf{ImageNette.} 
Beside the SimSiam backbone, we further apply our Prelax loss to  the BYOL framework \cite{BYOL} and evaluate the two methods on the ImageNette dataset. In Table \ref{tab:imagette-previous-methods}, Prelax also shows  a clear advantage over BYOL. Specifically, it improves the ResNet-18 version of BYOL by 0.7\%, and even outperforms the ResNet-50 version by 0.3\%.  

Here, we can see that Prelax yields significant improvement on two different datasets with two different backbone methods. Thus, Prelax could serve as a generic method for improving existing multi-view methods by encoding pretext-aware features into the residual relaxation. {For completeness, we also evaluate Prelax on the large scale dataset, ImageNet \cite{deng2009imagenet}, as well as its transferability to other kinds of downstream tasks, such as object detection and instance segmentation on MS COCO \cite{coco}.
As shown in Appendix \ref{sec:appendix-large-scale}, Prelax still consistently outperforms the baselines across all tasks. }

\begin{table}[t]
\caption{ A detailed comparison of SimSiam \cite{simsiam} and Prelax (ours) across three datasets: CIFAR-10 (C10), CIFAR-100 (C100), and  Tiny-ImageNet-200 (Tiny200) with the same hyper-parameters. }
\begin{subtable}{\textwidth}
    \centering
    \caption{In-domain linear evaluation.}
    \begin{tabular}{lccc}
    \toprule
        Method & CIFAR-10 & CIFAR-100 & Tiny-ImageNet-200 \\
    \midrule
    SimSiam \cite{simsiam} & 91.8 & 66.9 & 47.7 \\
    \midrule
    SimSiam + Prelax-std & 92.5 & 67.5 & 47.9 \\
    SimSiam + Prelax-rot & 92.4 & 67.3 & 47.1 \\
    SimSiam + Prelax-all & \textbf{93.4} & \textbf{70.0} & \textbf{49.2} \\
    \bottomrule
    \end{tabular}
    \label{tab:simsiam-in-domain-comparison}
\end{subtable}

\begin{subtable}{\textwidth}
    \centering
    \caption{Out-of-domain linear evaluation.}
    \begin{tabular}{lccc}
    \toprule
    Method & C100 $\to$ C10 & Tiny200 $\to$ C10 & Tiny200 $\to$ C100  \\
    \midrule
    SimSiam  \cite{simsiam} & 44.1 & 43.9 & 21.8 \\
    \midrule
    SimSiam + Prelax-std & \textbf{45.0} & \textbf{45.1} & 21.8 \\
    SimSiam + Prelax-rot & \textbf{45.0} & \textbf{45.1} & 22.0 \\
    SimSiam + Prelax-all & 44.9 & 44.6 & \textbf{22.1} \\
    \bottomrule
    \end{tabular}
    \label{tab:simsiam-transfer-learning}
\end{subtable}
\end{table}

\subsection{Effectiveness of Prelax Variants}

For a comprehensive comparison of the three variants of Prelax objectives (Prelax-std, Prelax-rot, and Prelax-all), we conduct controlled experiments on a range of datasets based on the SimSiam backbone. Except CIFAR-10, we also conduct experiments on CIFAR-100 and Tiny-ImageNet-200, which are more challenging with more classes of images. For a fair comparison, we use the same training and evaluation protocols across all tasks.

\textbf{In-domain Linear Evaluation.} As shown in Table \ref{tab:simsiam-in-domain-comparison}, our Prelax objectives outperform the multi-view objective consistently on all three datasets, where Prelax-all improves SimSiam by 1.6\% on CIFAR-10, 3.1\% on CIFAR-100, and 1.5\% on Tiny-ImageNet-200. Besides, Prelax-std and Prelax-rot are also better than SimSiam in most cases. Thus, the pretext-aware residuals in Prelax indeed help encode more useful semantics. 

\textbf{Out-of-domain Linear Evaluation.} Besides the in-domain linear evaluation, we also transfer the representations to a target domain. In the out-of-domain linear evaluation results shown in Table \ref{tab:simsiam-transfer-learning}, the Prelax objectives still have a clear advantage over the multi-view objective (SimSiam), while sometimes Prelax-std and Prelax-rot enjoy better transferred accuracy than Prelax-all.

\begin{table}[t]
\caption{Linear evaluation results of possible mechanisms for generalized multi-view learning on CIFAR-10 with SimSiam backbone. }
\begin{subtable}[h]{0.45\textwidth}
\centering
\caption{Comparison against alternative options.}
\begin{tabular}{lc}
\toprule
Method & Acc. (\%) \\
\midrule
SimSiam~\cite{simsiam} & 91.8 \\
SimSiam + margin loss  &  91.9 \\
\midrule
Rotation~\cite{rotation} & 88.3 \\
SimSiam + rotation aug.  & 87.9 \\
SimSiam + Rotation loss & 91.7 \\
\midrule
SimSiam + Prelax  (ours) & \textbf{93.4} \\
 \bottomrule
\end{tabular}
\label{tab:CIFAR-10-rotation-comparison}
\end{subtable}
\hfill
\begin{subtable}[h]{0.5\textwidth}
\centering
\caption{Ablation study.}
\begin{tabular}{lc}
\toprule
Method & Acc. (\%) \\
\midrule
Sim (\ie SimSiam~\cite{simsiam}) & 91.8 \\
\midrule
Sim + PL & 92.2 \\
Sim + R2S & 91.5 \\
R2S + PL  & 91.7 \\
Sim + PL + R2S (Prelax-std) & \textbf{92.5} \\
\midrule
Sim + RotPL & 91.1 \\
Sim + R3S & 91.9 \\
R3S + RotPL & 79.8 \\
Sim + RotPL + R3S (Prelax-rot) & \textbf{92.4} \\
\bottomrule
\end{tabular}
\label{tab:CIFAR-10-ablation-study}
\end{subtable}
\end{table}

\subsection{Empirical Understandings of Prelax}

\textbf{Comparison against Alternative Options.}  In Table \ref{tab:CIFAR-10-rotation-comparison}, we compare Prelax against several other relaxation options. ``SimSiam + margin'' refers to the margin loss discussed in Eq.~\eqref{eqn:margin-loss}, which uses a scalar $\eta$ to relax the exact alignment in multi-view methods. Here we tune the margin $\eta=0.5$ with the best performance. Nevertheless, it has no clean advantage over SimSiam. Then, we try several options for incorporating a strong augmentation and image rotation: 1) Rotation is the PL baseline by predicting rotation angles \cite{rotation}, which is inferior to multi-view methods (SimSiam). 2) ``SimSiam + rotation aug.'' directly applies a random rotation augmentation to each view, and learn with the SimSiam loss. However, it leads to lower accuracy, showing that the image rotation, as a strong augmentation, will \emph{hurt} the performance of multi-view methods. 3) ``SimSiam + Rotation'' directly combines the SimSiam loss and the Rotation loss for training, which is still ineffective. 4) Our Prelax shows a significant improvement over SimSiam and other variants, showing that the residual alignment is an effective mechanism for utilizing strong augmentations like rotation.

\textbf{Ablation Study.} We perform ablation study of each component of the Prelax objectives on CIFAR-10. From Table \ref{tab:CIFAR-10-ablation-study}, we notice that simply adding the PL loss alone cannot improve over SimSiam consistently, for example, Sim + RotPL causes 0.7 point drop in test accuracy. While with the help of our residual relaxation, we can improve over the baselines significantly and consistently, for example, Prelax-rot (Sim + RotPL + R3S) brings 0.6 point improvement on test accuracy. Besides, we can see that the PL loss is necessary by making the residual pretext-aware, without which the performance drops a lot, and the similarity constraint (Sim loss) is also important by avoiding bad cases when  augmented images drift far from the anchor image. 
Therefore, the ablation study shows the residual relaxation loss, similarity loss, and PL loss all matter in our Prelax objectives.

\begin{figure}[t]
    \centering
    \begin{subfigure}{0.45\textwidth}
    \includegraphics[width=\textwidth]{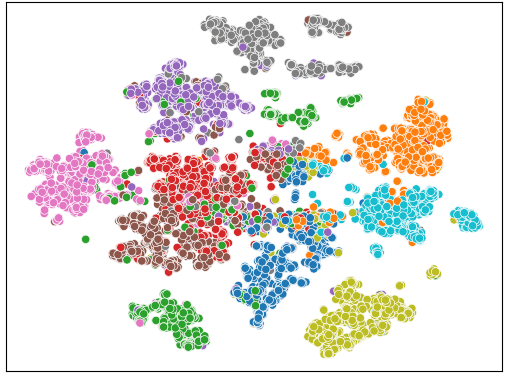}
    \caption{Representation visualization.}
    \label{fig:t-SNE}
    \end{subfigure}
    \hfill
    \begin{subfigure}{0.49\textwidth}
    \centering
    \includegraphics[width=\textwidth]{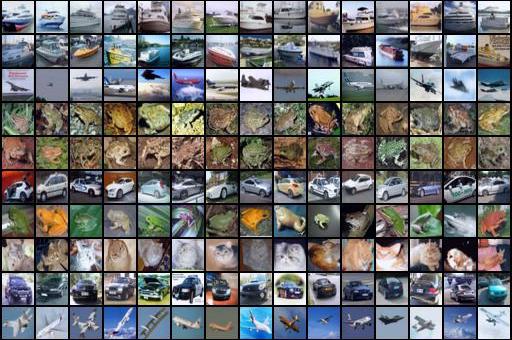}
    \caption{Nearest image retrieval.} 
    \label{fig:image-retrieval}
    \end{subfigure}
    \caption{(a) Representation visualization of our Prelax on CIFAR-10 test set. Each point represents an image representation and its color denotes the class of the image. 
    (b) On CIFAR-10 test set, given 10 random queries (not cherry-picked), we retrieve 15 nearest images in the representation space with Prelax (ours). }
    \vspace{-0.1in}
\end{figure}

\subsection{Qualitative Analysis}

\textbf{Representation Visualization.} To provide an intuitive understanding of the learned representations, we visualize them with t-SNE \cite{van2008visualizing} on Figure \ref{fig:t-SNE}. We can see that in general, our Prelax can learn well-separated clusters of representations corresponding to the ground-truth image classes.

\textbf{Image Retrieval.} In Figure \ref{fig:image-retrieval}, we evaluate Prelax on an image retrieval task. Given a random query image (not cherry-picked), the top-15 most similar images in representation space are retrieved, and the query itself is shown in the first column. 
We can see that although the unsupervised training with Prelax has no access to labels, the retrieved nearest images of Prelax are all correctly from the same class and semantically consistent with the query.

\section{Conclusion}

In this paper, we proposed a generic method, Prelax (Pretext-aware  Residual Relaxation), to account for the (possibly large) semantic shift caused by image augmentations. With  pretext-aware learning of the residual relaxation, our method generalizes existing multi-view learning by encoding both pretext-aware and pretext-invariant representations. Experiments show that our Prelax has outperformed existing multi-view methods significantly on a variety of benchmark datasets.

\section*{Acknowledgement}
Yisen Wang is partially supported by the National Natural Science Foundation of China under Grant 62006153, and Project 2020BD006 supported by PKU-Baidu Fund. 
Jiansheng Yang is supported by the National Science Foundation of China under Grant No. 11961141007.
Zhouchen Lin is supported by the NSF China under Grants 61625301 and 61731018, Project 2020BD006 supported by PKU-Baidu Fund, and Zhejiang Lab (grant no. 2019KB0AB02).

\bibliographystyle{plain}
\bibliography{main}

\appendix

\section{Experimental Details}
\label{sec:appendix-experimental-details}

\textbf{Evaluating Augmentations.} In Table 1, we compare different augmentations with a supervised ResNet-18 \cite{he2016deep} on CIFAR-10 test set. Specifically, we first train a state-of-the-art supervised ResNet-18 with 95.01\% test accuracy on CIFAR-10. \footnote{\url{https://github.com/kuangliu/pytorch-cifar}}. The supervised training uses no data augmentations. Afterwards, we evaluate the effect of different augmentations to the supervised model by applying each one (separately) to pre-process the test images of CIFAR-10. All of the included augmentations (except Rotation) belong to the augmentations used in SimSiam. For a fair comparison, we adopt the same configuration as in SimSiam and refer to the paper for more details. For Rotation, we adopt the same configuration as \cite{rotation}, where we sample a random rotation angle $\{0^\circ,90^\circ,180^\circ,270^\circ\}$ and use it to rotate the raw image clock-wise. 

\textbf{Data Augmentations and PL Targets.} We offer details of the augmentations by taking the SimSiam \cite{simsiam} variant of Prelax as an example. The BYOL \cite{BYOL} variants are implemented in the same way. 
For a fair comparison, we utilize the same augmentations in SimSiam \cite{simsiam}, while collecting the augmentation parameters as the target variables for our Predictive Learning (PL) objective in Prelax. We adopt the PyTorch notations for simplicity.  Specifically, for RandomResizedCrop, the operation randomly draws an $(i,j,h,k)$ pair, where $(i,j)$ denotes the center coordinates of the cropped region, while $(h,k)$ denotes the height and width of the cropped region. Accordingly,  we calculate the relative coordinates, the area ratio, and the aspect ratio (relative to the raw image), as four continuous target variables. Similarly, the ColorJitter opration randomly samples four factors corresponding to the adjustment for brightness, contrast, saturation, hue, respectively. We collect them as four additional continuous target variables.
As for operations like RandomHorizontalFlip, RandomGrayscale, RandomApply,  they draw a binary variable with $0/1$ outcome according to a predefined probability $p$, and apply the augmentations if it is 1 and do nothing otherwise. We collect these random outcomes $(0/1)$ as discrete target variables. As for the rotation operation, we take the rotation angles randomly drawn from the set $\{0^\circ,90^\circ,180^\circ,270^\circ\}$, as a discrete $4$-class categorical variable.

\section{Theoretical Results and Proofs}
\label{sec:appendix-proofs}

Here, we provide the complete proof of the theoretical results in Section 3.3. More rigorously, we give the definition of minimal and sufficient representations for self-supervision \cite{tsai2020self}, and give a more formal description of our results.
\begin{definition}[Minimal and Sufficient Representations for Signal $\bS$]
Let $\bZ^*$ be the minimal and sufficient representation for self-supervised signal $\bS$ if it satisfies the following conditions in the meantime: 1) $\bZ^*$ is sufficient,  $\bZ^*=\underset{\bZ}{\arg \max } I\left(Z; S\right)$; 2) $\bZ^*$ is minimal, \ie $\bZ^*=\underset{\bZ}{\operatorname{argmin}} H\left(\bZ|\bS\right)$.
\end{definition}
The following lemma shows that the maximal mutual information of $I(\bZ^*,\bS)$ is $I(\bX,\bS)$.
\begin{lemma}
For a minimal and sufficient representation $\bZ$ that is obtained with a deterministic encoder $\mathcal{F}_{\boldsymbol{\theta}}$ of input $\bX$ with enough capacity, 
we have $I(\bZ^*;\bS)=I(\bX;\bS)$.
\label{lemma:sufficiency}
\end{lemma}
\begin{proof}
As the encoder $\cF_{\btheta}$ is deterministic, it induces the following conditional independence: $\bS\indep\bZ\mid\bX$, which leads to the data processing Markov chain $\bS\leftrightarrow\bX\to\bZ$. Accordingly to the data processing inequality (DIP) \cite{cover1999elements}, we have $I(\bZ;\bS)\leq I(\bX;\bS)$, and with enough model capacity in $\cF_{\btheta}$, the sufficient and minimal representation $\bZ^*$ will have $I(\bZ^*;\bS)=\max_{\bZ}I(\bZ;\bS)=I(\bX;\bS)$.
\end{proof}
In the main text, we introduce several kinds of learning signals, the target variable $\bT$, the multi-view signal $\bS_v$, the predictive learning signal $\bS_a$, and the joint signal $(\bS_v,\bS_a)$ used by our Prelax method. For clarity, we denote the learned \emph{minimal and sufficient} representations as $\bZ_{\rm sup}$, $\bZ_{\rm mv}$, $\bZ_{\rm PL}$, $\bZ_{\rm Prelax}$, respectively.

Next, we restate Theorem 1 with the definitions above and provide a complete proof.
\begin{theorem}[restated]
We have the following inequalities on the four minimal and sufficient representations, $\bZ_{\rm sup}$, $\bZ_{\rm mv}$, $\bZ_{\rm PL}$, $\bZ_{\rm Prelax}:$
\begin{equation}
I(\bX;\bT)=I(\bZ_{\rm sup};\bT)\geq I(\bZ_{\rm Prelax};\bT)\geq\max(I(\bZ_{\rm mv};\bT),I(\bZ_{\rm PL};\bT)).
\end{equation}
\label{thm:task-information-appendix}
\end{theorem}
\begin{proof}
By Lemma \ref{lemma:sufficiency}, we have the following properties in the self-supervised representations:
\begin{equation}
    I(\bZ_{\rm mv};\bS_v)=I(\bX;\bS_v), \
    I(\bZ_{\rm PL};\bS_a)=I(\bX;\bS_a), \
    I(\bZ_{\rm Prelax};\bS_v,\bS_a)=I(\bX;\bS_v,\bS_a).
\end{equation}
Thus, for each minimal and sufficient self-supervised representation $\bZ\in\left\{\bZ_{\rm mv},\bZ_{\rm PL},\bZ_{\rm Prelax}\right\}$ and the corresponding signal $\bS\in\{\bS_v,\bS_a,(\bS_v,\bS_a)\}$, we have,
\begin{equation}
    I(\bZ;\bS;\bT)=I(\bX;\bS;\bT), \ I(\bZ;\bS|\bT)=I(\bX;\bS|\bT).
\end{equation}
Besides, because $\bZ$ is minimal, we also have,
\begin{equation}
    I(\bZ;\bT|\bS)\leq H(\bZ|\bS)=0.
\end{equation}
Together with the two equalities above, we further have the following equality on $I(\bZ;\bT)$:
\begin{equation}
\begin{aligned}
I(\bZ;\bT)&=I(\bZ;\bT;\bS) + I(\bZ;\bT|\bS)\\
&=I(\bX;\bT;\bS) + \underbrace{I(\bZ;\bT|\bS)}_0 \\
&=I(\bX;\bT)-I(\bX;\bT|\bS)\\
&=I(\bZ_{\rm sup};\bT)-I(\bX;\bT|\bS).
\end{aligned}
\label{eqn:self-sup-target-MI}
\end{equation}
Therefore, the gap between supervised representation $\bZ_{\rm sup}$ and each self-supervised representation $\bZ\in\left\{\bZ_{\rm mv},\bZ_{\rm PL},\bZ_{\rm Prelax}\right\}$ is $I(\bX;\bT|\bS)$, for which we have the following inequalities:
\begin{equation}
    \max(I(\bX;\bT|\bS_v),I(\bX;\bT|\bS_a))\geq \min(I(\bX;\bT|\bS_v),I(\bX;\bT|\bS_a))\geq
    I(\bX;\bT|\bS_v,\bS_a).
    \label{eqn:gap-inequalities}
\end{equation}
Further combining with Lemma \ref{lemma:sufficiency} and Eq.~\eqref{eqn:self-sup-target-MI}, we arrive at the inequalities on the target mutual information:
\begin{equation}
I(\bX;\bT)=I(\bZ_{\rm sup};\bT)\geq I(\bZ_{\rm Prelax};\bT)\geq\max(I(\bZ_{\rm mv};\bT),I(\bZ_{\rm PL};\bT)),
\end{equation}
which completes the proof.
\end{proof}
\textbf{Remark.}
Theorem \ref{thm:task-information-appendix} shows that the downstream performance gap between supervised representation $\bZ_{\rm sup}$ and self-supervised representation $\bZ$ is $I(\bX;\bT|\bS)$, \ie the information left in $\bX$ about the target variable $\bT$ except that in $\bS$. Thus, if we choose a self-supervised signal $\bS$ such that the left information is relatively small, we can guarantee a good downstream performance. Comparing the three self-supervised methods with learning signal $\bS_v$, $\bS_a$, and $(\bS_v,\bS_a)$, we can see that our Prelax utilizes more information in $\bX$, and consequently, the left information $I(\bX;\bT|\bS_v,\bS_a)$ is smaller than both multi-view methods $I(\bX;\bT|\bS_a)$ and predictive methods $I(\bX;\bT|\bS_a)$. 

In the following theorem, we further show that our Prelax has a tighter upper bound on the Bayes error of downstream classification tasks. 
To begin with, we prove a relationship between the supervised and self-supervised Bayes errors following \cite{tsai2020self}.
\begin{lemma}
Assume that $\bT$ is a $K$-class categorical variable. 
We define the Bayes error on downstream task $T$ as 
\begin{equation}
P^{e}:=\mathbb{E}_\bz\left[1-\max _{\bt \in \bT} P\left(\bT=\bt|\bz\right)\right].
\end{equation}
Denote the Bayes error of self-supervised learning (SSL) methods with signal $\bS$ as $P^e_{\rm ssl}$ and that of supervised methods as $P^e_{\rm sup}$. Then, we can show that the SSL Bayes error $P^e_{\rm ssl}$ can be upper bounded by the supervised Bayes error $P^e_{\rm sup}$, \ie
\begin{equation}
\bar P^e_{\rm ssl}\leq u^e :=\log 2 + P^e_{\rm sup}\cdot\log K + I(\bX;\bT|\bS).
\end{equation}
where $\bar P^e=\operatorname{Th}(P^e)=\min \{\max \{P^e, 0\}, 1-1 /K\}$ denotes the thresholded Bayes error in the feasible region, and $u^e$ denote the value of the upper bound. 
\label{eqn:general-upper-bound}
\end{lemma}
\begin{proof}
Denote the minimal and sufficient representations learned by SSL and supervised methods as $\bZ_{\rm ssl}$ and $\bZ_{\rm sup}$, respectively.
We use two following inequalities from \cite{feder1994relations} and \cite{cover1999elements},
\begin{gather}
    P^e_{\rm ssl}\leq-\log \left(1-P^e_{\rm ssl}\right) \leq H\left(\bT \mid \bZ_{\rm ssl}\right), \label{eqn:bayes-ssl}
    \\
    H(\bT|\bZ_{\rm sup})\leq \log2 + P^e_{\rm sup}\log K \label{eqn:bayes-sup}.
\end{gather}
Comparing $H(\bT|\bZ)$ and $H(\bT|\bZ_{\rm sup})$, together with Eq.~\eqref{eqn:self-sup-target-MI}, we can show that they are tied with the following equality,
\begin{equation}
\begin{aligned}
    H(\bT|\bZ_{\rm ssl})&=H(\bT)-I(\bZ_{\rm ssl};\bT)\\
    &=H(\bT)-I(\bZ_{\rm sup};\bT)+I(\bX;\bT|\bS)\\
    &=H(\bT|\bZ_{\rm sup})+I(\bX;\bT|\bS).
\end{aligned}
\end{equation}
Further combining Eq.~\eqref{eqn:bayes-ssl} \& \eqref{eqn:bayes-sup}, we have
\begin{equation}
\begin{aligned}
P^e_{\rm ssl}&\leq H\left(\bT \mid \bZ_{\rm ssl}\right)\\
&=H(\bT|\bZ_{\rm sup})+I(\bX;\bT|\bS)\\
&\leq\log2 + P^e_{\rm sup}\log K+I(\bX;\bT|\bS):=u^e,
\end{aligned}
\end{equation}
which completes the proof.
\end{proof}
Given the upper bound in Lemma \ref{eqn:general-upper-bound}, and the inequalities on the downstream performance gap $I(\bX;\bT|\bS)$ in Eq.~\eqref{eqn:gap-inequalities}, we will arrive at the following inequalities on the upper bounds on the self-supervised representations.
\begin{theorem}[restated]
We denote the the upper bound on the Bayes error of each representation, $\bZ_{\rm sup},\bZ_{\rm mv},\bZ_{\rm PL},\bZ_{\rm Prelax}$, by $u^e_{\rm sup},u^e_{\rm mv}, u^e_{\rm PL},u^e_{\rm Prelax}$, respectively. Then, they satisfy the following inequalities:
\begin{equation}
    u^e_{\rm sup}\leq u^e_{\rm Prelax} \leq \min(u^e_{\rm mv}, u^e_{\rm PL}) \leq \max(u^e_{\rm mv}, u^e_{\rm PL}).
\end{equation}
\label{thm:bayes-errors-appendix}
\end{theorem}
Theorem \ref{thm:bayes-errors-appendix} shows that our Prelax enjoys a tighter lower bounds on downstream Bayes error than both multi-view methods and predictive methods.

\section{Evaluation with Larger Backbone Networks} 
\label{sec:appendix-larger-backbone}

In the main text, we conduct experiments with the ResNet-18 backbone network. Here, for completeness, we further evaluate our Prelax with larger backbone networks. Specifically, for SimSiam variants, we evaluate the ResNet-34 \cite{he2016deep} across three datasets, CIFAR-10, CIFAR-100, and Tiny-ImageNet-200. For a fair comparison, we adopt the same hyper-parameters as for the ResNet-18 backbone. As can be seen for Table \ref{tab:simsiam-resnet-34},  all our Prelax variants achieves better results than the SimSiam baseline on all three datasets. Specifically, we can see that our Prelax-all variant attains the best results and it achieves better results with a larger backbone. Besides, we also experiment with ResNet-50 for the BYOL variant, where our Prelax variant also achieves better performance by improving from 92.3\% to 92.7\%.

\begin{table}[t]
    \centering
    \caption{Linear evaluation accuracy (\%) with ResNet-34 backbone.}
    \begin{tabular}{lccc}
    \toprule
        Method & CIFAR-10 & CIFAR-100 & Tiny-ImageNet-200 \\
    \midrule
    SimSiam \cite{simsiam} & 91.2 & 60.9 & 39.0 \\
    \midrule
    SimSiam + Prelax-std & 92.4 & 67.6 & 48.4 \\
    SimSiam + Prelax-rot & 93.0 & 67.0 & 40.9 \\
    SimSiam + Prelax-all & \textbf{93.9} & \textbf{69.3} & \textbf{49.4} \\
    \bottomrule
    \end{tabular}
    \label{tab:simsiam-resnet-34}
\end{table}

\begin{table}[t]\centering
    \caption{Evaluation of different pretraining methods on the downsampled ImageNet dataset (128x128) with ResNet-18 backbone. }
\begin{subtable}[h]{.27\textwidth}
    \centering
    \caption{Linear Evaluation.}
    \begin{tabular}{lc}
    \toprule
        Method & Acc (\%) \\
    \midrule
    BYOL & 49.2 \\
    Prelax (ours) & \textbf{51.1} \\
     \bottomrule
    \end{tabular}
\end{subtable}
\hfill
\begin{subtable}[h]{.38\textwidth}
    \centering
    \caption{Object Detection.}
\begin{tabular}{llll}
\toprule
 Method & AP$_{50}$ & AP       & AP$_{75}$ \\
\midrule
RandInit            & 32.7      & 19.5     & 20.1      \\
\midrule
BYOL                & 36.6      & 22.0     & 22.8      \\
Prelax (ours)             & \textbf{38.1}  & \textbf{23.3} & \textbf{23.9}  \\
\midrule
Supervised          & \textbf{39.4}  & \textbf{24.2} & \textbf{25.3} \\
\bottomrule
\end{tabular}
\end{subtable}
\hfill
\begin{subtable}[h]{.3\textwidth}
    \centering
    \caption{Instance Segmentation.}
\begin{tabular}{ll}
\toprule
 Method & MAP \\
\midrule
RandInit            & 15.8   \\
\midrule
BYOL                & 18.3      \\
Prelax (ours)       & \textbf{19.5}  \\
\midrule
Supervised          & \textbf{20.4} \\
\bottomrule
\end{tabular}
\end{subtable}
\label{tab:imagenet}
\end{table}

\section{Evaluation on Large Scale Datasets}
\label{sec:appendix-large-scale}

\textbf{Setup.} Although we cannot carry out the full ImageNet experiments with limited time and computation, we gather some preliminary results on the downsampled ImageNet dataset  (128x128) with the ResNet-18 backbone. For a fair comparison, our experiments are conducted with the official code of BYOL. All models are trained for 100 epochs with the default hyperparameters. 

\textbf{Evaluation protocol.} For downstream evaluation, we report both the linear evaluation task on ImageNet and two transfer learning tasks on the MS COCO dataset \cite{coco}. Specifically, we perform object detection on the standard RetinaNet \cite{retinanet} with FPN \cite{fpn}, and conduct instance segmentation on the standard Mask R-CNN \cite{maskrcnn} with FPN \cite{fpn}. We compare the performances of models initialized with different pretrained weights on COCO:
\begin{itemize}
    \item \textbf{RandInit}: randomly initialized weights;
    \item \textbf{BYOL}: unsupervised pretrained weights with BYOL;
    \item \textbf{Prelax} (ours): unsupervised pretrained weights with Prelax;
    \item \textbf{Supervised} (oracle): supervised pretrained weights.
\end{itemize}

From Table \ref{tab:imagenet}, we can see that even on the large-scale dataset, our Prelax still has a clear advantage over BYOL on all downstream tasks, including both in-domain linear evaluation and out-of-domain instance segmentation and object detection tasks.

\section{Sensitivity Analysis of Prelax Coefficients}
\label{sec:appendix-sensitivity}
Here we provide a detailed discussion on the effect of each coefficient of our Prelax objectives. We adopt the default hyper-parameters unless specified. For Prelax-std, it has three coefficients, the R2S interpolation coefficient $\alpha$,  the similarity loss coefficient $\beta$, and the predictive loss coefficient $\gamma$. From Figure \ref{fig:Prelax-std-coefficients},  we can see that a positive $\alpha$ introduces certain degree of residual relaxation to the exact alignment and help improve the downstream performance. The best accuracy is achieved with a medium $\alpha$ at around 0.5. In addition, a large similarity coefficient $\beta$ tends to yield better performance, showing the necessity of the similarity constraint. Nevertheless, too large $\beta$ can also diminish the effect of residual relaxation and leads to slight performance drop. At last, a positive PL coefficient $\gamma$ is shown to yield better representations, although it might lead to representation collapse if it is too large, \eg $\gamma>0.5$.

\begin{figure}[h]
    \centering
    \begin{subfigure}{\linewidth}
    \includegraphics[width=\linewidth]{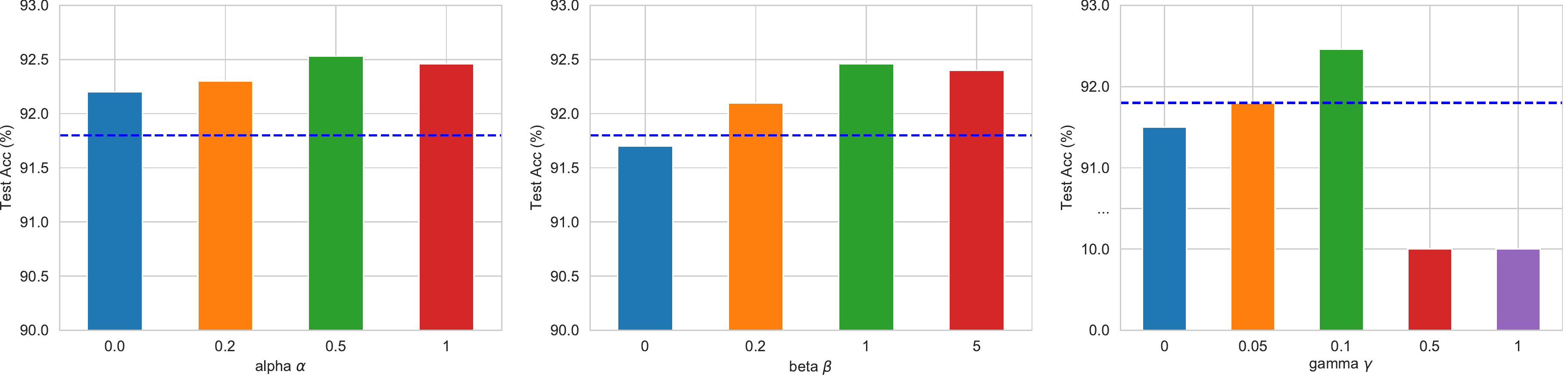}
    \caption{Prelax-std.}
    \label{fig:Prelax-std-coefficients}
    \end{subfigure}
    \begin{subfigure}{\linewidth}
    \includegraphics[width=\linewidth]{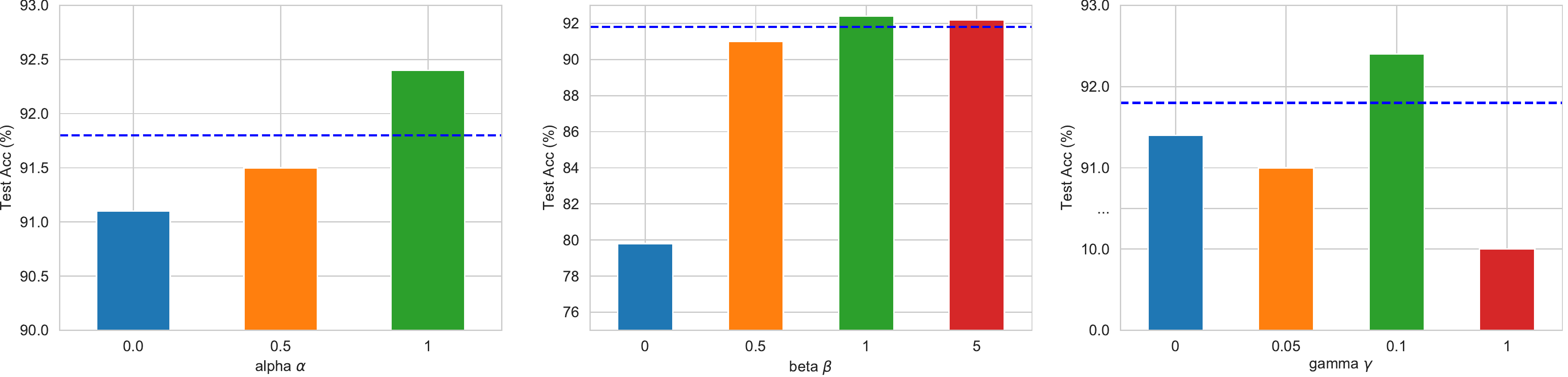}
    \caption{Prelax-rot.}
    \label{fig:Prelax-rot-coefficients}
    \end{subfigure}
    \caption{Linear evaluation results of different Prelax-std and Prelax-rot coefficients on CIFAR-10 with SimSiam backbone. The dashed blue line denotes the result of the SimSiam baseline. }
    \label{fig:my_label}
\end{figure}

For Prelax-rot, as shown in Figure \ref{fig:Prelax-rot-coefficients}, the behaviors of $\beta$ and $\gamma$ are basically consistent with Prelax-std. Nevertheless, we can see that only $\alpha=1$ can yield better results than the SimSiam baseline, while other alternatives cannot. This could be due to the fact that the residual relaxation involves the first view $\bx_1$ and its rotation-augmented view $\bx_3$, and the R3S loss is designed between $\bx_3$ and the second view $\bx_2$. Therefore, in order to align $\bx_3$ and $\bx_2$ like the alignment between $\bx_1$ and $\bx_2$, all the relaxation information in $\bx_3$ (which $\bx_1$ does not have) must be accounted for, which corresponds to $\alpha=1$ in R3S loss. We show that incorporating the rotation information in this way will indeed richer representation semantics and better performance.

Besides, we also find that in certain cases, adopting a reverse residual $\br_{21}$ in the R2S loss can bring slightly better results. In Figure \ref{fig:reverse-residual}, we investigate this phenomenon by comparing the normal and reverse residuals in R2S loss (applied for Prelax-std and Prelax-all) and R3S loss (applied for Prelax-rot). We can see that for R2S loss, using a reverse residual improves the accuracy by around 0.3 point, while for R3S loss, the reverse residual leads to dramatic degradation. This could be due to that R2S relaxes the gap between $\bx_1$ and $\bx_2$, whose representations are learned through swapped prediction in SimSiam's dual objective. Thus, we might also need to swap the direction of the residual to be consistent. Instead, in R3S, the relaxation is crafted between $\bx_1$ and $\bx_3$, so we do not need to swap the direction.
Last but not least, we note that with the normal residual, Prelax-std and Prelax-all still achieve significantly better results than the SimSiam baseline, and the reverse residual can further improve on it.

\begin{figure}[h]\centering
\includegraphics[width=.4\linewidth]{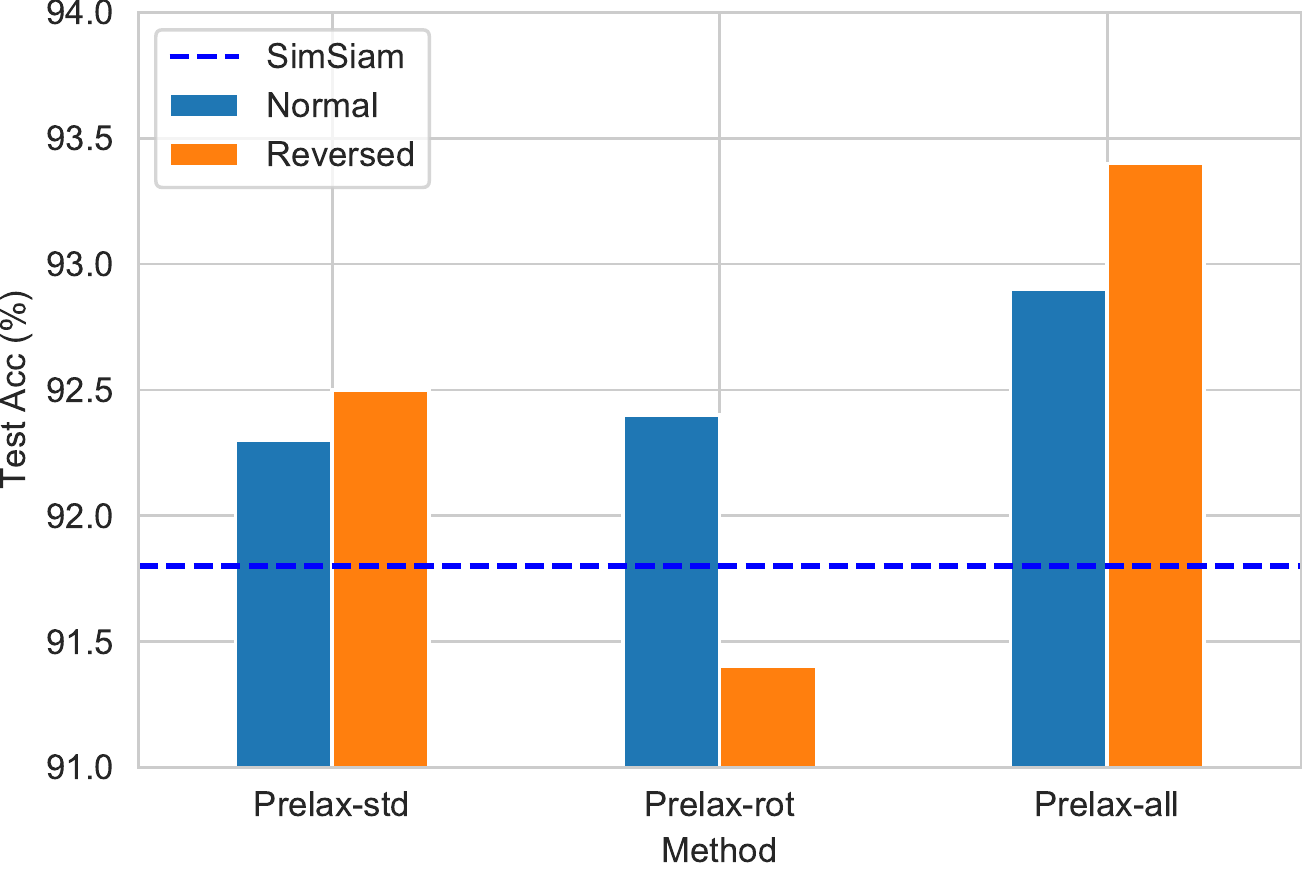}
\caption{Comparison of normal and reverse residuals for Prelax variants on CIFAR-10 with SimSiam backbone.}
\label{fig:reverse-residual}
\end{figure}

\section{Learning Dynamics} 
\label{sec:appendix-learning-dynamics}

In Figure \ref{fig:learning-dynamics}, we compare SimSiam with Prelax-rot in terms of the learning dynamics. We can see that with our residual relaxation technique, both the relaxation loss and the similarity loss become larger than SimSiam. In particular, the size of the residual indeed converges to a large value with Prelax (1.1) than with SimSiam (0.6). As for the downstream classification accuracy, we notice that Prelax-rot starts with a lower accuracy, but converges to a large accuracy at last. This indicates that Prelax-rot learns to encode more image semantics, which may be harder to learn at first,  but will finally outperform the baseline with better representation ability.

\begin{figure}[h]
    \centering
    \includegraphics[width=\linewidth]{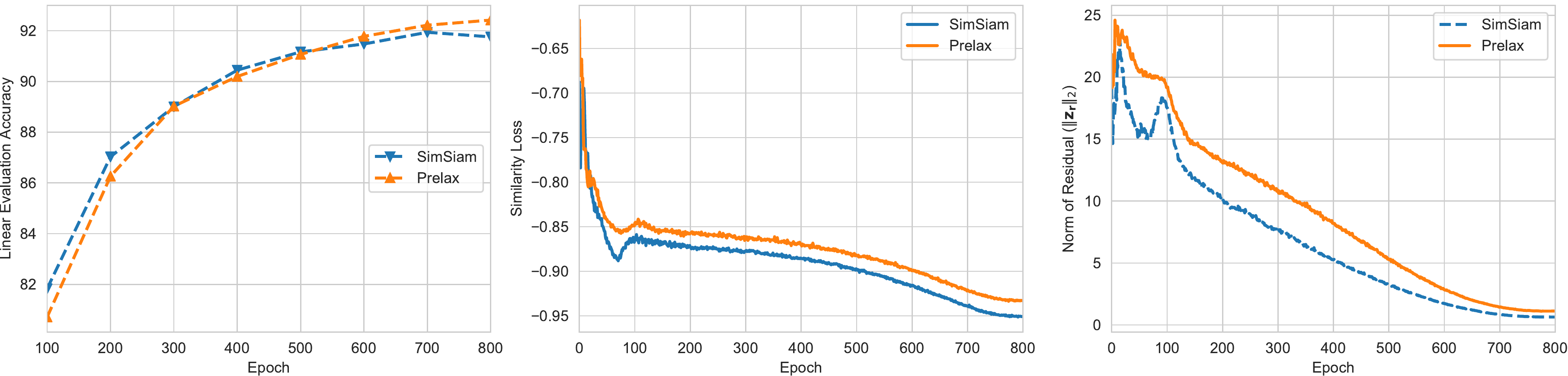}
    \caption{A comparison of learning dynamics between SimSiam \cite{simsiam} and Prelax (ours) on CIFAR-10. Left: linear evaluation accuracy (\%) on the test set per epoch. Middle: similarity loss per epoch. Right: norm of the residual vector (\ie $\Vert\br_{31}\Vert_2$) per epoch. }
    \label{fig:learning-dynamics}
\end{figure}

\end{document}